\documentclass{article} 
\usepackage{arxiv}
\usepackage{natbib}
\usepackage{hyperref}
\usepackage{url}
\usepackage[table]{xcolor}
\usepackage{xspace}
\usepackage{multirow}
\usepackage{tcolorbox}

\usepackage[utf8]{inputenc} %
\usepackage[T1]{fontenc}    %
\usepackage{url}            %
\usepackage{booktabs}       %
\usepackage{amsfonts}       %
\usepackage{nicefrac}       %
\usepackage{microtype}      %
\usepackage{amssymb} 
\usepackage{amsmath,csquotes, array,amsthm,thmtools,amsbsy}
\usepackage{xparse}
\usepackage{footmisc}
\usepackage{thm-restate}
\usepackage{mathtools}
\usepackage{tabularx}
\usepackage[shortlabels]{enumitem}

\theoremstyle{plain}

\newtheorem{proposition}{Proposition}

\theoremstyle{definition}

\theoremstyle{remark}

\usepackage{graphicx}
\usepackage{float}
\usepackage{subcaption}
\graphicspath{{figures/}}
\newlength{\smallfigwidth}
\newlength{\smallfigheight}
\newlength{\smallfigsep}
\newlength{\legendheight}
\def\v+#1{\ensuremath{\mathbf{#1}}\xspace}
\def\m+#1{\ensuremath{\mathbf{#1}}\xspace}

\title{Integrating Local and Global Entropy for Uncertainty Quantification in LLMs}

\author{ 
    {Johanne Medina} \\
	Qatar Computing Research Institute, HBKU\\
	Doha, Qatar\\
	\texttt{jomedina@hbku.edu.qa} \\
    \And
    {Tianyi Zhou}\\
	KTH Royal Institute of Technology\\
	Stockholm, Sweden \\
	\texttt{tzho@kth.se} \\
	\And
    {Keivin Isufaj} \\
	Qatar Computing Research Institute, HBKU\\
	Doha, Qatar\\
	\texttt{keisufaj@hbku.edu.qa} \\
    \And
    {Aristides Gionis} \\
	KTH Royal Institute of Technology\\
	Stockholm, Sweden \\
	\texttt{argioni@kth.se} \\
    \And
	{Sanjay Chawla} \\
	Qatar Computing Research Institute, HBKU\\
	Doha, Qatar\\
	\texttt{schawla@hbku.edu.qa} \\
}

\begin{document}

\maketitle

\begin{abstract}
Large language models hallucinate confidently, making uncertainty
quantification (UQ) essential for reliable deployment. Existing methods rely
predominantly on token-level signals, leaving the geometric structure of intermediate hidden states underused. In this paper, we take the geometric complexity of hidden-state matrices as a measure of the \emph{global} uncertainty of LLMs, while treating token-level uncertainty estimation as a \emph{local} metric. We show that hidden-state geometric entropy (\emph{global} uncertainty) and token-level entropy (\emph{local} uncertainty) are statistically near-orthogonal, capturing distinct failure regimes for reliability prediction. In particular, global geometry recovers the confident-but-wrong failure mode that local signals systematically miss.
Building on this, we propose \textbf{Global-Local Uncertainty (GLU)}, an
unsupervised, single-pass score that fuses the two signals via a multiplicative gate. Across three model families and six benchmarks, GLU matches or outperforms all unsupervised baselines while requiring only a single forward pass and remaining length-normalized and architecture-agnostic. Code is available on \url{https://github.com/qcri/GLU.git}.

\end{abstract}

\section{Introduction}
Hallucination remains one of the most persistent failure modes of large language 
models (LLMs). Despite rapid advances in capability, frontier systems continue to 
produce fluent, specific, and incorrect answers. A recent cross-domain benchmark 
spanning 6{,}000 questions across 42 topics found that fewer than 1\% of evaluated 
models score above zero on a $[-100, 100]$ reliability index, with the best-performing 
frontier model reaching only 33 \citep{jackson2025aaomniscience}. Uncertainty 
quantification (UQ) offers a principled path forward, and a wide range of methods 
have been proposed, from token-level entropy \citep{zhang2025sledselflogitsevolution} 
and sampling-based consistency \citep{farquhar2024detecting, yadkori2024conformalabstention} 
to evidential \citep{sensoy2018evidentialdeeplearningquantify, ma2025estimating} and 
attention-based approaches \citep{sriramanan2024llm, skean2025layer}; yet few are 
deployed in practice.  While other methods impose at least one of the following barriers of multiple 
generation passes, task-specific supervision, sensitivity to response length, or 
architectural assumptions, we argue that a practical UQ score should be \emph{unsupervised}, 
computed in a \emph{single forward pass}, \emph{length-normalized}, and 
\emph{architecture-agnostic}. 

The linear representation hypothesis \citep{park2023linear} posits that high-level 
concepts are encoded as directions in the hidden-state space of LLMs. When a model 
retrieves a well-encoded fact, its hidden states move in a tight, consistent direction 
implying that the geometry of retrieval is compact. When it lacks reliable factual grounding, the hidden states wander through loosely related directions as the model generates 
together a plausible-sounding reply. Consider a model asked for the birth year of 
a figure it has not reliably encoded. It may produce a confident four-digit number 
while its representations drift through adjacent concepts rather than converging on 
a specific memory. We observe that this wandering is measurable. Motivated by \citet{skean2025layer}, 
who shows that intermediate hidden layers encode richer representations than the final 
layer and that information-theoretic geometry metrics effectively characterize 
representation quality, we quantify this drift via the entropy of eigenvalue distribution from the hidden-state trajectory. This geometric signal is computed directly from the hidden states of 
standard greedy generation with no additional forward pass.

The distinction between the embedding (hidden state vector space) and unembedding layers (next-token prediction) is not merely 
architectural. \citet{park2023linear} show that the embedding space obeys a 
non-Euclidean geometry in which concepts are encoded as directions, and that this 
structure is fundamentally different from the output space where token probabilities 
are computed. This separation provides a principled basis for why token-level entropy 
and hidden-state geometric complexity measure different things: one operates in the space of next-token distributions, the other in the space of semantic directions, and the two need not agree.  
Since the two failure modes arise from structurally distinct layers, we term them 
\emph{global} (hidden-state geometric complexity) and \emph{local} (token-level 
entropy) uncertainty, and ask how best to combine them. Figure~\ref{fig:motivation} 
confirms that both signals are individually informative yet neither is sufficient. Correct responses cluster tightly in the low local-uncertainty region while incorrect responses concentrate at high global uncertainty, but the low-$x$ region where token-level methods declare confidence still contains a non-trivial fraction of incorrect responses that token entropy alone cannot recover. Jointly, the two signals 
span a far wider discriminative range than either achieves alone. We study additive 
and multiplicative fusion strategies across architectures and tasks, finding that 
multiplicative combination consistently dominates.


\textbf{This paper makes the following contributions: }(1) We extract complementary uncertainty  signals from two distinct representational layers: a global geometric signal from 
the embedding layer via matrix Rényi entropy of hidden-state trajectories, and a 
local token-level signal from the unembedding layer via Shannon and evidential 
epistemic entropy. Together, they provide a principled, multi-view perspective on 
LLM uncertainty that goes beyond logit-only approaches. (2) We show empirically 
that the two signals are statistically near-orthogonal and cover distinct uncertainty 
regimes. In particular, the global signal recovers the \emph{confident-but-wrong} 
failure mode that local signals systematically miss. (3) We propose 
Global-Local Uncertainty (\textbf{GLU}) , a lightweight multiplicative fusion of the two signals 
that requires no labeled data, no additional forward passes, and is normalized for 
response length. (4) We validate GLU across three model families and six 
benchmarks, and conduct comprehensive ablation studies on the choice of local and 
global uncertainty estimators, demonstrating both the effectiveness of the framework 
and the individual contribution of each component.

\section{Related Work}

\paragraph{Local and output-distributional uncertainty.}
A common approach estimates LLM reliability from generation-time signals such as
sequence likelihood, token entropy, model-elicited confidence, or probabilistic
views of next-token inference \citep{kadavath2022language, dalal2024,
ma2025estimating}. LogTokU is especially relevant as it argues that probabilities
alone can lose evidence-strength information thus instead estimates token-level
aleatoric and epistemic uncertainty directly from logits, building on the broader
evidential-learning view that uncertainty can be represented through evidence over
class probabilities \citep{sensoy2018evidentialdeeplearningquantify, ma2025estimating}. Zhou et al.
further show that token-level uncertainty is useful for confabulation detection,
using it to select and aggregate hidden states for response-level reliability
prediction \citep{zhou2026can}. Other methods use local distributional features
during decoding, such as layer-wise logit contrast, or generate multiple samples
and measure agreement or semantic dispersion
\citep{zhang2025sledselflogitsevolution, manakul2023selfcheckgpt,
farquhar2024detecting, nikitin2024kernel, yadkori2024conformalabstention}.
These approaches are effective and often minimally intrusive, but they primarily
measure uncertainty in tokens or sampled outputs rather than the global coherence
of the model's internal trajectory.

\paragraph{Representation-based and global uncertainty.}
A complementary line of work studies whether reliability is encoded in hidden
states, layer activations, representation geometry, or attention behavior.
Layer-by-Layer shows that intermediate layers can contain especially informative
representations and analyzes them using information-theoretic, geometric, and
invariance-based measures \citep{skean2025layer}. RAUQ similarly moves beyond
output probabilities by identifying uncertainty-aware attention heads whose
attention patterns correlate with incorrect generations, producing efficient
sequence-level uncertainty estimates without task-specific labels
\citep{vazhentsev2026efficient}. More broadly, probing and representation-based
methods use hidden states or activations to predict factuality, confabulation, or
response reliability
\citep{li2024confidencemattersrevisitingintrinsic, orgad2024llms, zhou2026can,
sriramanan2024llm}. These methods capture response-level failures that token
probabilities can miss, but representation-only signals may overlook local
knowledge gaps already visible in the output distribution.

\paragraph{Supervised detectors and calibration dependence.}
Several hallucination detectors learn reliability decision boundaries from labeled
or pseudo-labeled generations
\citep{li2024confidencemattersrevisitingintrinsic, orgad2024llms, zhou2026can,
du2024haloscope}. While effective in-distribution, these methods depend on
calibration data and may be sensitive to shifts across datasets, model families,
or generation settings.We include a supervised-detector comparison in
Appendix~\ref{app:supervisedvs}.

GLU instead combines local and global uncertainty without learning a detector.
It uses hidden-state geometry to amplify token-level uncertainty when the response
trajectory is diffuse, capturing both local uncertainty and response-level drift
without labeled calibration data, pseudo-labels, or retraining.
\section{Preliminary}
\label{sec:preliminary}
We formally introduce the notation of this paper. Let $\v+x$ denote a prompt and $\v+y=(y_1,\dots,y_T)$ denote a generated response of length $T$, where each token $y_t$ is drawn from a vocabulary $\mathcal{V}$ of size $V$. At generation step $t$ and transformer layer $\ell$, the language model produces a hidden-state vector $\v+h_t^{(\ell)} \in \mathbb{R}^d$. The model then autoregressively predicts the next token using the final-layer representation. Specifically, the final hidden state is mapped to the logit vector
\[
\v+z_t = \m+W_U \v+h_t^{(L)} \in \mathbb{R}^V,
\]
where $\m+W_U$ is the unembedding matrix. The conditional distribution over the next token is given by
\[
p(y_t \mid \v+x, \v+y_{<t}) = \mathrm{softmax}(\v+z_t).
\]

Next, we collect hidden states across all $T$ generation steps at layer $\ell$ into the \emph{representation matrix}
\begin{equation}
    \m+H^{\ell} = \begin{bmatrix} \v+h_1^{\ell)} \\ \vdots \\ \v+h_T^{(\ell)} \end{bmatrix}
    \in \mathbb{R}^{T \times d},
\end{equation}
where each row encodes the model's internal state at one generation step.
More concretely, $\m+H$ encodes the full \emph{trajectory} of the response in representation space, the sequence of rows traces how the model's internal representation evolves as it generates each token. We simply write $\m+H$ when there is no ambiguity. 

We define the \emph{Gram matrix} of the response as
\begin{equation}
    \m+K  = \m+H \m+H^\top \in \mathbb{R}^{T \times T},
\end{equation}
whose $(i,j)$-th entry $\m+K_{ij} = \langle \v+h_i^{(L)}, \v+h_j^{(L)} \rangle$
measures the similarity between the representations at steps $i$ and $j$.
Intuitively, $\m+K$ captures the geometric structure of the response
trajectory: when the model generates a coherent response, the hidden states tend to concentrate along a few dominant
directions, making $\m+K$ low-rank; when the response is uncertain or incoherent, the hidden states scatter across many directions, making $\m+K$ closer to a scaled identity.
\section{Global and Local Uncertainty of LLMs}
\label{sec:method}
Uncertainty in LLM generation appears at two complementary levels. At the
\emph{token level}, the model may be uncertain over plausible next tokens,
capturing local aleatoric or epistemic uncertainty. At the \emph{response level},
the hidden-state trajectory $\m+H$ may become geometrically complex, spreading
across multiple directions in latent space where this structure is captured by the
eigenvalues and eigenvectors of its representation geometry \cite{skean2025layer}. These signals are
not redundant; we observe that a response can have confident token predictions while its global trajectory drifts in Figure \ref{fig:motivation}, producing a plausible but hallucinated answer. We therefore
measure both token-level uncertainty (local) and geometric entropy (global).

\begin{figure*}[t]
  \centering
  \includegraphics[width=0.8\linewidth]{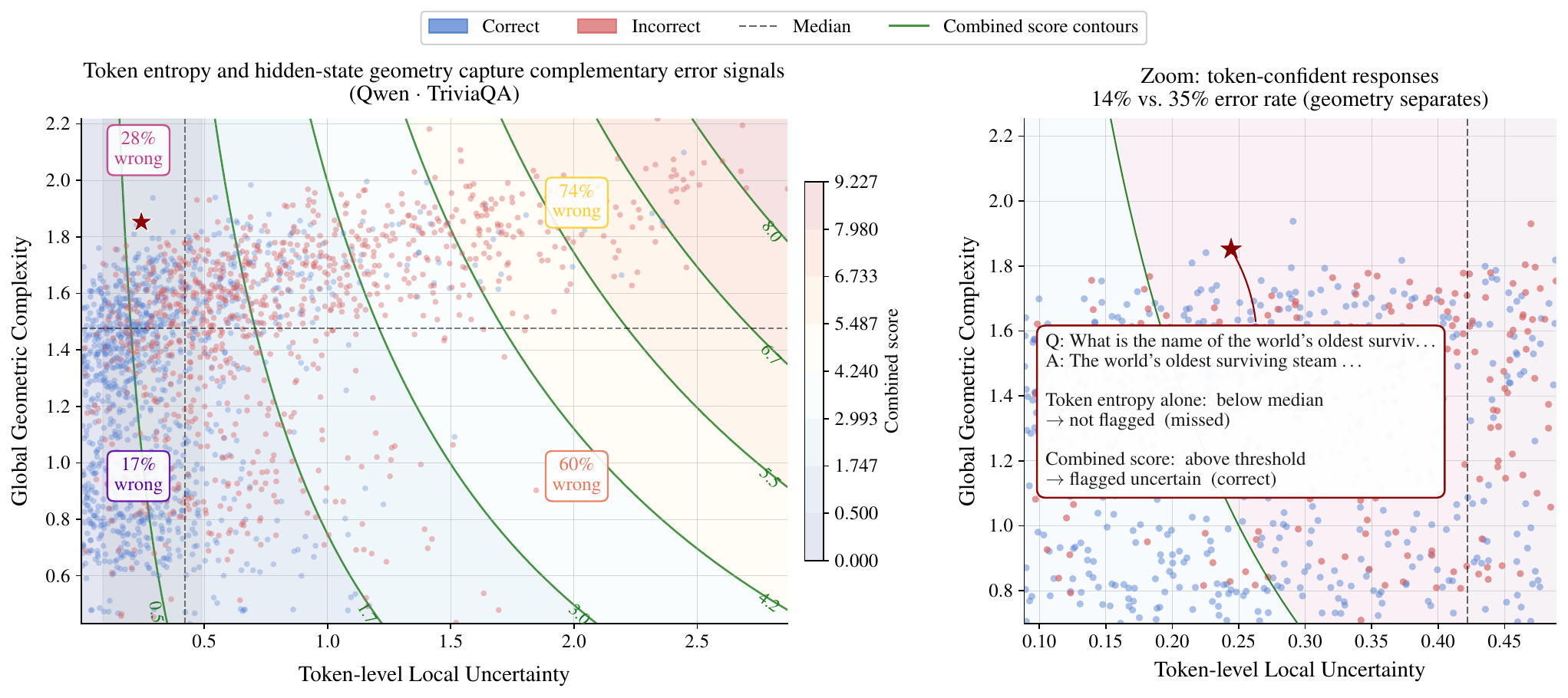}
  \caption{\textbf{Local and global signals capture complementary uncertainty
  information.} Each point is a response from Qwen~2.5-7B on TriviaQA
  (blue = correct, red = incorrect). $x$-axis: mean Shannon entropy over the
  most uncertain tokens (local). $y$-axis: geometric complexity of the
  hidden-state trajectory (global). Contours show the product of the two
  signals, used only for visualization. \textit{Left:} token entropy provides
  the primary separation, while geometry further discriminates within each
  regime; quadrant error rates confirm the two views jointly span a wider
  uncertainty range than either alone. \textit{Right (zoom):} among
  below-median-entropy responses, geometry still separates many incorrect
  ones. The starred example, low local but high global uncertainty, is the
  failure mode that motivates fusing both views.}
  \label{fig:motivation}
\end{figure*}

\subsection{Local Uncertainty: Token-level Entropy}
We first introduce token-level uncertainty measures to characterize local uncertainty during generation. Following the evidential deep learning framework~\cite{sensoy2018evidentialdeeplearningquantify,ma2025estimating}, we decompose uncertainty into an aleatoric component (AU), capturing distributional ambiguity, and an epistemic component (EU), capturing lack of evidence.

For each generated token $y_t$, let $\tau_{(1)},\dots,\tau_{(k)}$ denote the indices of the top-$k$ vocabulary candidates ranked by logit value. We define a truncated predictive distribution over these candidates by renormalizing their logits:
\begin{align}
p_j = \frac{e^{z_{\tau_{(j)}}}}{\sum_{j'=1}^k e^{z_{\tau_{(j')}}}}.
\end{align}
The aleatoric uncertainty is then defined as the Shannon entropy of this truncated distribution:
\begin{align}
\mathrm{AU}(t)
= -\sum_{j=1}^k p_j \log_2 p_j.
\label{eq:shannon_au}
\end{align}

To estimate epistemic uncertainty, we follow the evidential interpretation of logits as evidence supporting different candidate tokens. Specifically, we map each logit to a positive Dirichlet concentration parameter using the softplus transformation:
\begin{align}
\alpha_j = \ln(1 + e^{z_{\tau_{(j)}}}).
\end{align}
Compared with the ReLU mapping used in prior work~\cite{ma2025estimating}, the softplus function guarantees strictly positive concentration parameters, which is consistent with the Dirichlet parameterization. We therefore define epistemic uncertainty as the inverse Dirichlet strength:
\begin{align}
\mathrm{EU}(t)
= \frac{k}{\sum_{j=1}^k (\alpha_j + 1)},
\label{eq:logtoku_eu}
\end{align}
where the denominator denotes the total Dirichlet strength. Intuitively, larger total evidence implies lower epistemic uncertainty. 

A token is considered locally uncertain when both its predictive distribution is diffuse (high AU) and its supporting evidence is weak (high EU). We combine these two complementary factors using a multiplicative interaction score:
\begin{align}
R(t) = -\mathrm{AU}(t)\cdot \mathrm{EU}(t).
\end{align}

Averaging uncertainty uniformly across all generated tokens can dilute salient uncertain positions due to low-information or highly predictable tokens (e.g., stop words or formatting tokens). We therefore focus on the $m$ most uncertain positions. Let $\mathcal{W}_m$ denote the set of token indices corresponding to the top-$m$ values of $R(t)$. We define the aggregated local uncertainty score as
\begin{align}
\bar{R}
= \frac{1}{m}\sum_{t\in\mathcal{W}_m} R(t).
\end{align}

\subsection{Global Uncertainty: R\'enyi Entropy of Hidden-State Gram Matrices}
 
We define \emph{geometric entropy} as the R\'enyi entropy of the normalized
eigenspectrum of the hidden-state Gram matrix, quantifying the effective
dispersion of representations across latent geometric directions.
 
We characterize global uncertainty through the geometric complexity of
hidden-state representations.  As established in
Section~\ref{sec:preliminary}, the Gram matrix
$\m+K^{(\ell)}=\m+H^{(\ell)}{\m+H^{(\ell)}}^{\!\top}\in\mathbb{R}^{T\times T}$
encodes the pairwise similarity structure of the response trajectory at
layer $\ell$.  Its eigenvalues $\lambda_1\ge\cdots\ge\lambda_T\ge 0$
characterize how the trajectory's variance is distributed across orthogonal
directions in representation space: a sharply peaked spectrum reflects a
low-rank, coherent trajectory, while a flat spectrum reflects a diffuse,
uncertain one.  We now convert this spectral structure into a single scalar
uncertainty score.
 
\paragraph{Entropy of the normalized spectrum.}
Normalizing the eigenvalues as
$\tilde{\lambda}_i=\lambda_i/\mathrm{tr}(\m+K^{(\ell)})$ yields a
probability distribution over principal directions.  We measure its
dispersion via the order-$\alpha$ R\'enyi entropy of that distribution~\cite{skean2025layer}:
\begin{equation}\label{eq:renyi}
  S_\alpha(\m+K)
  \;=\;
  \frac{1}{1-\alpha}
  \log\!\Bigl(\textstyle\sum_{i=1}^{T}\tilde{\lambda}_i^{\,\alpha}\Bigr),
  \qquad \alpha>0,\;\alpha\neq 1.
\end{equation}

We use $\alpha=2$ (collision entropy), which admits the closed form

\begin{equation}\label{eq:s2}
  S_2(\m+K)
  \;=\;
  -\log\mathrm{tr}\!\left(\!\left(\frac{\m+K}{\mathrm{tr}(\m+K)}\right)^{\!2}\right)
  \;=\;
  -\log\frac{\|\m+K\|_F^2}{\mathrm{tr}(\m+K)^2},
\end{equation}

and therefore requires only the squared Frobenius norm of $\m+K$ rather than
its full eigendecomposition.  Forming $\m+K=\m+H{\m+H}^{\!\top}$ costs
$O(T^2 d)$, after which $\|\m+K\|_F^2$ is computed in $O(T^2)$; the
eigendecomposition alternative would cost $O(T^3)$.  This advantage is
substantial in practice since $d\gg T$ for large language models (hidden
dimension $d$ in the thousands, response length $T$ in the tens to
hundreds).  Low $S_2$ indicates that the trajectory concentrates along a
few dominant directions (coherent generation); high $S_2$ indicates a
diffuse trajectory (geometrically incoherent generation). We give the detailed proof in Appendix \ref{app:entropy-proof}.
 
\paragraph{Normalization for cross-response and cross-model comparability.}
We adopt two normalizations to make $\tilde{S}$ comparable across responses and models.
First, $S_2(\m+K^{(\ell)})$ grows with $T$ and attains its maximum of $\log T$ when the spectrum is uniform ($\tilde{\lambda}_i=1/T$); we therefore divide by $1+\log T$ so the normalized score lies in $[0,1)$ regardless of response length.
Second, because representation quality varies with depth and no single layer is uniformly informative across architectures, we average the length-normalized entropy over all $L$ layers to obtain
\begin{equation}\label{eq:stilde}
  \tilde{S}
  \;=\;
  \frac{1}{L}\sum_{\ell=1}^{L}
  \frac{S_2\!\left(\m+K^{(\ell)}\right)}{\log T}
  \;\in\; [0,1].
\end{equation}
This layer-agnostic formulation avoids model-specific layer selection and yields a single, response-level uncertainty score computable in one forward pass.


\subsection{Global-Local Fusion}
\label{sec:glu}
The two signals cover distinct regimes (Fig.~\ref{fig:motivation}):
correct responses cluster at low local uncertainty but spread broadly along
$\tilde{S}$, while incorrect responses include a \emph{confident-but-wrong}
subset which has low token entropy yet elevated geometric complexity, that token-level
detectors miss entirely. We fuse the two via a multiplicative gate:
\begin{equation}
\label{eq:glu}
\boxed{\;\mathrm{GLU} \;=\; (1+\tilde{S})\,\bar{R}\;}
\end{equation}
The factor $(1+\tilde{S})$ acts as a geometric amplifier, such that, when the trajectory
is diffuse, local uncertainty is up-weighted; when compact, GLU reduces to the
local signal alone. This directly targets the confident-but-wrong failure mode,
which additive fusion cannot isolate because it treats the global signal as a
fixed offset rather than a modulator.
\section{Experiments and Ablation Studies}
\label{sec:experiments}

In this section, we evaluate GLU across six benchmarks and three model families. To empirically validate the design choices motivated in Section~\ref{sec:glu}, we conduct comprehensive ablation studies over local and global uncertainty components, isolating the contribution of each and confirming that the multiplicative geometric–probabilistic combination outperforms alternatives.

\paragraph{Datasets and Models.}
We evaluate on six benchmarks spanning factuality, knowledge retrieval, mathematical reasoning, Arabic QA, long-form generation, and multi-turn reasoning:
\textit{TruthfulQA}~\cite{lin2022truthfulqa}, 
\textit{TriviaQA}~\cite{joshi2017triviaqa}, 
\textit{MATH}~\cite{hendrycks2021math}, 
\textit{ArabicaQA}~\cite{abdallah2024arabicaqa}, 
\textit{LongForm}~\cite{koksal2023longform}, and the \textit{GSM8K multi-turn} setting from \citet{laban2025llms}. Further information on the composition of the datasets are provided in Appendix~\ref{app:datasets}.
We evaluate three instruction-tuned model families:
\textit{Qwen2.5-7B}~\cite{qwen2025qwen25}, 
\textit{Gemma3-12B}~\cite{gemma2025gemma3}, and 
\textit{Fanar1-9B}~\cite{fanar2025}, 
an Arabic-centric Arabic--English model continually pretrained from Gemma-9B, allowing us to test whether GLU remains effective across different model families, languages, and latent-space configurations.

\paragraph{Evaluation protocol.}
We adopt greedy decoding to generate responses for each model--benchmark pair, and store the corresponding hidden states, logits, and probability distributions. All uncertainty estimation methods are applied post hoc to this shared set of model outputs, ensuring a fair and controlled comparison.
We report two complementary evaluation metrics.
AUROC measures how well an uncertainty score discriminates between correct and incorrect responses (higher is better; $0.5$ corresponds to random guessing). The prediction rejection ratio (PRR) is a more recent metric that quantifies
how well an uncertainty score supports selective abstention; we use it because it is the primary
metric adopted by RAUQ~\citep{vazhentsev2026efficient}, and we refer the reader to
\citet{malinin2021uncertainty} for its full definition. The two metrics capture complementary aspects of performance: AUROC reflects discriminative capability, whereas PRR assesses whether the uncertainty ranking is effective for selective prediction.

\paragraph{Baselines and Ablations.}
The baselines cover the major families of unsupervised uncertainty estimation methods, including LogProb\cite{yadkori2024iterativeprompting} and LogTokU (local token-level methods) \cite{ma2025estimating}, P(True) (self-evaluation) \cite{kadavath2022language}, and RAUQ (attention-based uncertainty estimation) \cite{vazhentsev2026efficient}.
The GLU ablations isolate the contribution of individual components in Eq.~\ref{eq:glu}. Full breakdown of the ablations is summarize in Table~\ref{tab:methods_summary}.

\begin{table*}
\centering
\scriptsize
\caption{Methods and ablations. For GLU variants, the final score combines a global geometric
uncertainty term $S$ with a local token-level uncertainty term $\bar{u}$ as
$\mathrm{GLU}(x,y)=(1+S)\bar{u}$, where larger uncertainty corresponds to a more negative score.}
\label{tab:methods_summary}
\setlength{\tabcolsep}{4pt}
\renewcommand{\arraystretch}{1.12}
\begin{tabularx}{\linewidth}{@{}l c c X@{}}
\toprule
\textbf{Method} & \textbf{Global} & \textbf{Local} & \textbf{Explanation} \\
\midrule

\multicolumn{4}{@{}l}{\textit{Baselines}} \\[2pt]

LogProb
  & --
  & $\frac{1}{T}\sum_{t=1}^{T}\log p(y_t \mid y_{<t},x)$
  & Average token log-probability of the generated response. \\

LogTokU
  & --
  & $-\mathrm{AU}_{\mathrm{EDL}}(t)\cdot \mathrm{EU}(t)$
  & Token-level uncertainty baseline using Dirichlet aleatoric uncertainty
    and epistemic uncertainty~\citep{ma2025estimating}. \\

P(True)
  & $p(\text{``True''}\mid x,y)$
  & --
  & The model self-evaluates whether its generated answer is
    true~\citep{kadavath2022language}. \\

RAUQ
  & $\mathrm{RAUQ}(x,y)$
  & --
  & Recurrently fuses uncertainty-aware attention-head activations with
    token probabilities; peaks over a middle-layer
    subset~\citep{vazhentsev2026efficient}. \\

\midrule
\multicolumn{4}{@{}l}{\textit{Ablations on Eq.\ref{eq:glu}}} \\[2pt]

GLU
  & $\tilde{S}=\dfrac{1}{L}\!\sum_{\ell=1}^{L}\dfrac{S_\alpha^\ell}{1+\log T}$
  & $\bar{u}_{\mathrm{ShE}\times\mathrm{EU}}
      = \frac{1}{k}\!\sum_{t \in \mathcal{K}}
        {-\mathrm{SE}(t)\,\mathrm{EU}(t)}$
  & Length-normalised matrix R\'enyi-2 entropy averaged across all $L$
    layers, combined with the mean of the $k$ most uncertain local
    scores. \\

GLU-EDL
  & $\tilde{S}$
  & $\bar{u}_{\mathrm{EDL}}
      = \frac{1}{k}\!\sum_{t \in \mathcal{K}}
        {-\mathrm{AU}_{\mathrm{EDL}}(t)\,\mathrm{EU}(t)}$
  & Replaces the local ShETokU term with the LogTokU-style Dirichlet
    token uncertainty. \\

GLU-$\bar{S}$
  & $\bar{S}_\alpha=\dfrac{1}{L}\!\sum_{\ell=1}^{L} S_\alpha^\ell$
  & $\bar{u}_{\mathrm{ShE}\times\mathrm{EU}}$
  & Removes the $1/(1{+}\log T)$ length normalisation; both $\tilde{S}$
    and $\bar{S}_\alpha$ average R\'enyi entropy across all $L$
    layers. \\

GLU-$\tilde{S}^*$
  & $\max_{\ell}\,\tilde{S}^{\ell}$
  & $\bar{u}_{\mathrm{ShE}\times\mathrm{EU}}$
  & Selects the single most uncertain length-normalised layer instead of
    averaging over all $L$ layers as in $\tilde{S}$. \\

GLU-SP
  & $\tilde{S}$
  & $\bar{u}_{\mathrm{ShE}\times\mathrm{EU}^{\mathrm{sp}}}
      = \frac{1}{k}\!\sum_{t \in \mathcal{K}}
        {-\mathrm{SE}(t)\,\mathrm{EU}^{\mathrm{sp}}(t)}$
  & Uses softplus evidence $\alpha_k\!=\!\mathrm{softplus}(\ell_k)+\varepsilon$
    instead of $\mathrm{ReLU}(\ell_k)+1$ in the Dirichlet epistemic
    uncertainty term. \\

GLU-DK
  & $\tilde{S}$
  & $\bar{u}_{\mathrm{ShE}\times\mathrm{EU}}^{\,k'}$,\;
    $k'=\left\lfloor k/(1+\tilde{S})\right\rfloor$
  & Uses an adaptive local window size, shrinking the number of selected
    uncertain tokens as global geometric uncertainty increases. \\

GLU-AU
  & $\tilde{S}$
  & $\bar{u}_{\mathrm{ShE}}
      = \frac{1}{k}\!\sum_{t \in \mathcal{K}}{-\mathrm{SE}(t)}$
  & Drops the epistemic factor entirely; retains only the local Shannon
    entropy term amplified by $\tilde{S}$. \\

GLU-EU
  & $\tilde{S}$
  & $\bar{u}_{\mathrm{EU}}
      = \frac{1}{k}\!\sum_{t \in \mathcal{K}}{-\mathrm{EU}(t)}$
  & Drops Shannon entropy entirely; retains only the Dirichlet epistemic
    uncertainty term amplified by $\tilde{S}$. \\

GLU-$S_\alpha$-AU
  & $\bar{S}_\alpha$
  & $\bar{u}_{\mathrm{ShE}}
      = \frac{1}{k}\!\sum_{t \in \mathcal{K}}{-\mathrm{SE}(t)}$
  & Jointly ablates length normalisation and the epistemic factor;
    tests whether raw multi-layer geometry pairs best with Shannon-only
    local uncertainty. \\

GLU-$S_\alpha$-SP
  & $\bar{S}_\alpha$
  & $\bar{u}_{\mathrm{ShE}\times\mathrm{EU}^{\mathrm{sp}}}
      = \frac{1}{k}\!\sum_{t \in \mathcal{K}}
        {-\mathrm{SE}(t)\,\mathrm{EU}^{\mathrm{sp}}(t)}$
  & Tests whether softplus evidence pairs better with unnormalised
    depth-averaged entropy than with $\tilde{S}$. \\

GLU-$S^*$
  & $S_\alpha^*=\max_{\ell} S_\alpha^\ell$
  & $\bar{u}_{\mathrm{ShE}\times\mathrm{EU}}$
  & Uses the single most uncertain raw layer instead of the depth-averaged
    global entropy. \\

GLU-EU-SP
  & $\tilde{S}$
  & $\bar{u}_{\mathrm{EU}^{sp}}
      = \frac{1}{k}\sum_{t \in \mathcal{K}}-\mathrm{EU}^{sp}(t)$
  & Uses only the softplus-based evidential uncertainty term, dropping the
    Shannon entropy factor. \\

GLU-$S_\alpha$-SP-EU
  & $\bar{S}_\alpha$
  & $\bar{u}_{\mathrm{EU}^{sp}}
      = \frac{1}{k}\sum_{t \in \mathcal{K}}-\mathrm{EU}^{sp}(t)$
  & Combines raw depth-averaged matrix R\'enyi entropy with the softplus-based
    evidential-only local score. \\
\midrule
\multicolumn{4}{@{}l}{\textit{Additive fusion variants}} \\[2pt]

Add-$S_\alpha$
  & $\bar{S}_\alpha$
  & $\bar{u}_{\mathrm{ShE}\times\mathrm{EU}}$
  & Uses additive fusion, $\bar{S}_\alpha+\bar{u}_{\mathrm{ShE}\times\mathrm{EU}}$,
    instead of the multiplicative GLU interaction. \\

Add-$\tilde{S}$
  & $\tilde{S}$
  & $\bar{u}_{\mathrm{ShE}\times\mathrm{EU}}$
  & Uses additive fusion, $\tilde{S}+\bar{u}_{\mathrm{ShE}\times\mathrm{EU}}$,
    with length-normalized geometry. \\
\bottomrule
\end{tabularx}
\end{table*}

\begin{table}[t]
\centering
\scriptsize
\setlength{\tabcolsep}{4pt}
\renewcommand{\arraystretch}{1.08}
\caption{Response reliability estimation across three models and six datasets, 
evaluated by AUROC and PRR. \textbf{LogTokU} captures a purely \emph{local} signal 
from token-level output logits; \textbf{RAUQ} captures a \emph{global} signal from 
intermediate hidden-state and attention representations; \textbf{GLU} forms their 
\emph{multiplicative} fusion, achieving the best score in 13/18 cells on AUROC and 
PRR and ranking first or second in the remaining cells.
\textbf{Bold}: best per row; \underline{underline}: second best.}
\label{tab:main_results}
\begin{tabular}{ll ccc ccc}
\toprule
& & \multicolumn{3}{c}{\textbf{AUROC}} & \multicolumn{3}{c}{\textbf{PRR}} \\
\cmidrule(lr){3-5} \cmidrule(lr){6-8}
\textbf{Model} & \textbf{Dataset} 
  & \textbf{LogTokU} & \textbf{RAUQ} & \textbf{GLU}
  & \textbf{LogTokU} & \textbf{RAUQ} & \textbf{GLU} \\
\midrule
\multirow{6}{*}{\textbf{Qwen}}
& TruthfulQA & \textbf{0.680} & 0.582 & \underline{0.616} & \textbf{0.377} & 0.063 & \underline{0.195} \\
& TriviaQA   & 0.786 & \underline{0.789} & \textbf{0.834} & \underline{0.606} & 0.581 & \textbf{0.717} \\
& MATH       & 0.595 & \underline{0.699} & \textbf{0.767} & 0.216 & \underline{0.424} & \textbf{0.537} \\
& ArabicaQA  & 0.521 & \textbf{0.610} & \underline{0.556} & $-$0.028 & \textbf{0.186} & \underline{0.140} \\
& Longform   & \underline{0.490} & 0.469 & \textbf{0.496} & $-$0.045 & \underline{$-$0.044} & \textbf{$-$0.040} \\
& Multiturn  & 0.523 & \underline{0.644} & \textbf{0.739} & 0.008 & \underline{0.189} & \textbf{0.407} \\
\midrule
\multirow{6}{*}{\textbf{Gemma}}
& TruthfulQA & \underline{0.531} & 0.466 & \textbf{0.548} & \textbf{0.115} & $-$0.094 & \underline{0.087} \\
& TriviaQA   & 0.597 & \underline{0.700} & \textbf{0.817} & 0.224 & \underline{0.416} & \textbf{0.698} \\
& MATH       & \textbf{0.801} & 0.771 & \underline{0.799} & \underline{0.693} & 0.639 & \textbf{0.699} \\
& ArabicaQA  & 0.538 & \underline{0.642} & \textbf{0.647} & 0.099 & \underline{0.277} & \textbf{0.322} \\
& Longform   & 0.343 & \underline{0.354} & \textbf{0.358} & $-$0.249 & \textbf{$-$0.234} & \underline{$-$0.242} \\
& Multiturn  & 0.457 & \textbf{0.734} & \underline{0.701} & $-$0.274 & \textbf{0.506} & \underline{0.423} \\
\midrule
\multirow{6}{*}{\textbf{Fanar}}
& TruthfulQA & 0.579 & \textbf{0.641} & \textbf{0.641} & 0.192 & \underline{0.282} & \textbf{0.307} \\
& TriviaQA   & 0.706 & \textbf{0.821} & \textbf{0.821} & 0.445 & \underline{0.712} & \textbf{0.730} \\
& MATH       & 0.664 & \underline{0.779} & \textbf{0.788} & 0.324 & \underline{0.624} & \textbf{0.627} \\
& ArabicaQA  & 0.490 & \textbf{0.613} & \underline{0.563} & $-$0.025 & \textbf{0.239} & \underline{0.143} \\
& Longform   & \underline{0.569} & 0.549 & \textbf{0.654} & \underline{0.077} & 0.066 & \textbf{0.200} \\
& Multiturn  & 0.463 & \underline{0.804} & \textbf{0.831} & $-$0.166 & \underline{0.654} & \textbf{0.687} \\
\bottomrule
\end{tabular}
\end{table}

\begin{table}[t]
\centering
\small
\setlength{\tabcolsep}{7pt}
\renewcommand{\arraystretch}{1.14}
\caption{Mean performance across all 18 model--dataset settings, ranked among 14 methods.
The \textbf{Space} column denotes the signal each method uses: L (local, token-level) and
G (global, hidden-state geometry); $\mathrm{L}\times\mathrm{G}$ marks multiplicative fusion
and $\mathrm{L}{+}\mathrm{G}$ additive fusion. GLU ranks first on both mean AUROC ($0.676$)
and mean PRR ($0.369$), above the global-signal baseline RAUQ (rank $12$) and the local-signal
baseline LogTokU (rank $14$). The multiplicative ablations cluster just below it (ranks
$2$--$9$), while the additive variant Add-$\tilde{S}$ falls to rank $13$, isolating the
multiplicative fusion as the source of the gain. Full results and rankings are in
Appendix~\ref{app:methods_and_ablations}.}
\label{tab:mean_all_methods}
\vspace{4pt}
\begin{tabular}{lllcccc}
\toprule
\textbf{Group} & \textbf{Method} & \textbf{Space} & \textbf{AUROC} & \textbf{Rank} & \textbf{PRR} & \textbf{Rank} \\
\midrule
Proposed & \cellcolor{blue!8}\textbf{GLU} & $\mathrm{L}\times\mathrm{G}$ & 0.676 & 1 & 0.369 & 1 \\
\midrule
Baseline & LogProb & L & 0.655 & 10 & 0.326 & 10 \\
Baseline & LogTokU & L & 0.574 & 14 & 0.144 & 14 \\
Baseline & P(true) & G & 0.654 & 11 & 0.311 & 11 \\
Baseline & RAUQ & G & 0.648 & 12 & 0.305 & 12 \\
\midrule
Ablation & GLU-SP & $\mathrm{L}\times\mathrm{G}$ & 0.675 & 3 & 0.367 & 2 \\
Ablation & GLU-$\tilde{S}^*$ & $\mathrm{L}\times\mathrm{G}$ & 0.674 & 4 & 0.366 & 3 \\
Ablation & GLU-DK & $\mathrm{L}\times\mathrm{G}$ & 0.675 & 2 & 0.366 & 4 \\
Ablation & GLU-$\bar{S}$ & $\mathrm{L}\times\mathrm{G}$ & 0.674 & 5 & 0.365 & 5 \\
Ablation & GLU-$S_\alpha$-SP & $\mathrm{L}\times\mathrm{G}$ & 0.673 & 6 & 0.363 & 6 \\
Ablation & GLU-AU & $\mathrm{L}\times\mathrm{G}$ & 0.671 & 7 & 0.362 & 7 \\
Ablation & GLU-$S_\alpha$-AU & $\mathrm{L}\times\mathrm{G}$ & 0.667 & 8 & 0.357 & 8 \\
Ablation & GLU-$S^*$ & $\mathrm{L}\times\mathrm{G}$ & 0.665 & 9 & 0.351 & 9 \\
Ablation & Add-$\tilde{S}$ & $\mathrm{L}{+}\mathrm{G}$ & 0.621 & 13 & 0.259 & 13 \\
\bottomrule
\end{tabular}
\end{table}


\subsection{The Complementary Error Signals}
\label{sec:complementary}
We demonstrate that token-level and global uncertainty signals cover distinct failure 
regimes and together reduce the \emph{confident-but-wrong} failure mode that neither can 
address alone. Figure~\ref{fig:motivation} illustrates this complementarity for 
Qwen~2.5-7B on TriviaQA.

\paragraph{Token-level uncertainty provides strong but incomplete separation.}
The left panel plots each response along two axes: aggregated token-level entropy 
(local uncertainty, $x$-axis) and hidden-state geometric complexity (global 
uncertainty, $y$-axis). Correct responses (blue) cluster tightly in the low 
local-uncertainty region yet spread broadly along the $y$-axis; incorrect responses 
(red) concentrate at high global uncertainty yet spread comparably along the $x$-axis. 
Both signals are individually informative, but neither is sufficient. Crucially, the 
low-$x$ region, where token-level methods declare confidence, still contains a 
non-trivial fraction of incorrect responses that token entropy alone cannot separate 
from correct ones. The error rates across quadrants confirm that the two signals jointly 
span a far wider discriminative range than either achieves alone.

\paragraph{Global geometry recovers confident-but-wrong responses.}
The most consequential failure mode for uncertainty-based error detection is a 
\emph{confident-but-wrong} response \cite{zhou2026can}, where the model generates an incorrect answer while 
maintaining low token-level entropy throughout. These responses occupy the upper-left 
region of the red cluster with low token uncertainty yet elevated geometric complexity. 
Here, the model commits to each token with strong logit evidence, but its hidden-state 
trajectory drifts across many representational directions, signaling latent indecision 
invisible to token-level probes. The starred example makes this concrete: its token 
entropy falls below the median and would go unflagged by any token-level detector, yet 
its elevated geometric complexity pushes the combined score above the decision threshold, 
correctly identifying it as uncertain.

\subsection{Main Results}
Table~\ref{tab:main_results} reports AUROC and PRR for all methods across three 
models and six datasets; we discuss each metric in turn.

GLU achieves the best AUROC in 13 of 18 settings and ranks second in the remainder, confirming that multiplicative fusion of complementary 
signals consistently outperforms either component in isolation. The pattern directly 
validates the complementarity hypothesis from Figure~\ref{fig:motivation}. GLU improves over the local-only baseline LogTokU by up to 22.0\% 
(Gemma-TriviaQA: 0.817 vs.\ 0.597) and over the global-only baseline RAUQ by up to 
11.7\% (Gemma-TriviaQA: 0.817 vs.\ 0.700). 

RAUQ overtakes GLU mainly in settings where the local token-level signal is weak:
ArabicaQA for Qwen and Fanar, and multiturn MATH for Gemma. In these cases,
LogTokU is near chance in AUROC and often has negative PRR, suggesting that token
confidence is poorly aligned with correctness. Since GLU multiplicatively combines
the local and global signals, an uninformative local component can attenuate an
otherwise strong global signal, allowing the pure-global RAUQ baseline to edge
ahead. However, GLU usually remains close to RAUQ in these
cases, indicating that the fusion degrades gracefully rather than failing outright.

The PRR results show that GLU's advantage carries over to selective prediction. GLU achieves the best PRR in 12 of 18 settings and the highest mean PRR overall, indicating that its uncertainty rankings place unreliable responses closer to the rejection tail more effectively than either LogTokU or RAUQ alone. This is important because the practical goal is not only to distinguish correct from incorrect responses, but also to prioritize which outputs should be deferred, flagged, or verified first.

\subsection{Ablation Study}
Table~\ref{tab:mean_all_methods} aggregates performance across all 18 settings and 
reveals three findings. First, GLU ranks first on both mean AUROC (0.676) and mean 
PRR (0.369), confirming that multiplicative fusion of global and local signals 
consistently outperforms any individual component: RAUQ ranks 12th and LogTokU 
14th, showing that neither signal alone is sufficient. Second, all multiplicative 
ablations cluster tightly at ranks 2--9,
indicating that the gain is robust to the choice of local and global uncertainty 
estimator. The performance gain is driven by the fusion principle rather than the specific 
choice of estimator. Third, and most decisively, the 
additive variant Add-$\tilde{S}$ drops to rank 13 (0.621 AUROC, 0.259 PRR), falling 
below all baselines except LogTokU. This sharp degradation isolates multiplicative 
fusion as the source of the gain: treating the global signal as a fixed offset rather 
than a modulator of local uncertainty fails to capture the interaction between the 
two signals that GLU is designed to exploit.

\begin{figure}[t]
  \centering
  \includegraphics[width=\linewidth]{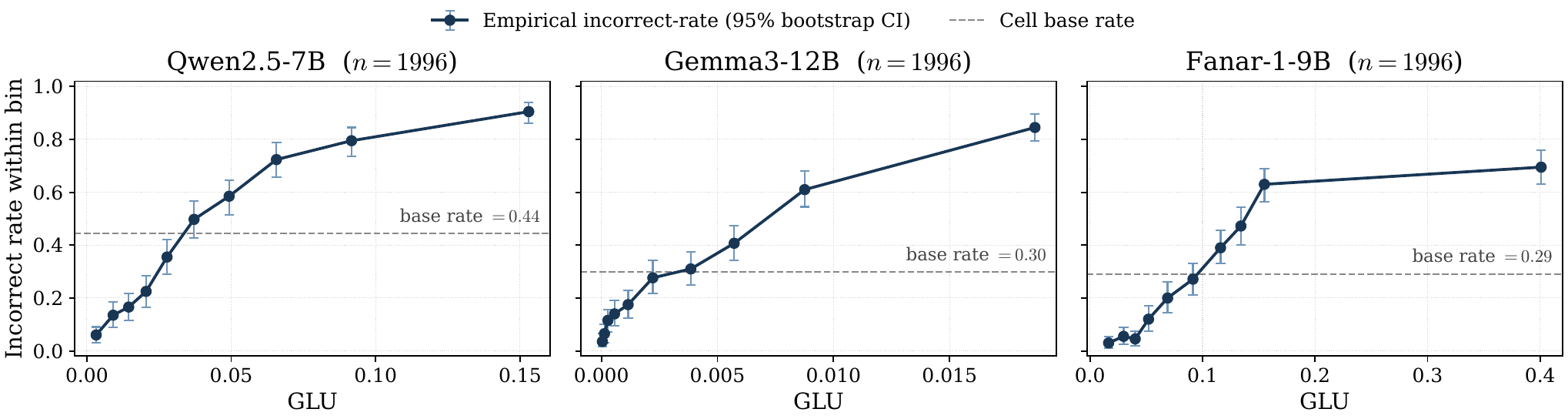}
  \caption{Binned reliability of GLU on TriviaQA. Responses are sorted by
$-\mathrm{GLU}$ and partitioned into deciles; markers report the empirical
incorrect-rate per bin with $95\%$ bootstrap CIs ($1{,}000$ resamples).
Dashed line: cell-level base incorrect-rate. GLU yields a monotonic increase in
incorrect-rate across confidence deciles for all three models, indicating a graded
reliability signal rather than only separating extreme cases.}
  \label{fig:reliability}
\end{figure}

\subsection{GLU on TriviaQA}
\label{sec:reliability}

We further analyze whether GLU provides a calibrated reliability signal on
TriviaQA, beyond its aggregate discrimination performance. For each model, we
sort responses by increasing GLU uncertainty and partition them into deciles. We
then measure the empirical incorrect rate within each bin.

Figure~\ref{fig:reliability} shows a consistent monotonic trend across all three
models: responses assigned higher GLU uncertainty are more likely
to be incorrect. The gap between the lowest- and highest-uncertainty bins reaches
approximately $84$ percentage points for Qwen, $83$ percentage points for Gemma,
and $66$ percentage points for Fanar. These ranges are well above the model-level
base incorrect rates, indicating that GLU does not simply reflect dataset-level
difficulty or a binary correct/incorrect separation.

This result suggests that GLU captures a graded notion of answer reliability.
Low-GLU responses are mostly correct, while high-GLU responses concentrate a
large fraction of errors. Therefore, GLU can support reliability-aware generation
settings in which answers are selectively trusted, flagged, or routed for
additional verification.

\section{Conclusion}

We presented a multi-view perspective on LLM uncertainty grounded in the structural 
distinction between the embedding and unembedding layers. From this perspective, we 
extracted two complementary signals: a global geometric uncertainty from hidden-state 
trajectories, quantified via matrix Rényi entropy, and a local token-level uncertainty 
from the unembedding layer, based on Shannon and evidential epistemic entropy. We 
showed empirically that the two signals are complementary and cover 
distinct failure regimes, with the global signal recovering the confident-but-wrong 
responses that token-level methods systematically miss. Building on this 
complementarity, we proposed \textbf{GLU}, a lightweight multiplicative fusion that 
requires no labeled data, no additional forward passes, and is normalized for response 
length. Across three model families and three benchmarks, GLU matches or outperforms 
all unsupervised baselines and remains competitive with supervised methods that lack 
cross-dataset generalization. Ablation studies further confirm that the performance 
gain is robust to the choice of local and global uncertainty estimators, suggesting 
that the complementarity between the two layers is a structural property of LLMs 
rather than an artifact of any particular design choice.


\paragraph{Limitations.} 
GLU requires access to hidden states and token-level output distributions, making it most directly applicable to open-weight models. While GLU achieves the best mean AUROC and PRR across our evaluation, the relative strength of global and local signals varies by task; in some settings, RAUQ remains highly competitive, suggesting that adaptive fusion could further improve robustness. Future work should explore principled layer selection, adaptive weighting, and validation on larger models, additional languages, and domain-specific deployments.

\bibliographystyle{unsrtnat}
\bibliography{references}

\clearpage
\appendix



\begin{table}[t]
\centering
\caption{\textbf{Classifier-based:} AUROC for response reliability estimation across three benchmarks and three model
families. \textbf{Bold}: best per column per model. \underline{Underline}: second best. Although classifier-based methods output better scores, they are not generalizable. 
\textit{Italic}: third best.}
\label{tab:classifier_results}
\small
\begin{tabular}{llccc}
\toprule
\textbf{Model} & \textbf{Method} & \textbf{TruthfulQA} & \textbf{TriviaQA} & \textbf{Math} \\
\midrule
\multirow{6}{*}{\textbf{Qwen}}
    & \textbf{Prob(Exact)}      & 0.758 & 0.781 & 0.627 \\
    & \textbf{Probe(EOS)}       & \textbf{0.794} & \underline{0.812} & 0.703 \\
    & \textbf{Probe(EU)}        & 0.759 & 0.754 & 0.646 \\
    & \textbf{Probe(AVG)}       & \underline{0.761} & 0.786 & 0.699 \\
    & \textbf{RespEntropy}      & 0.705 & 0.793 & \textbf{0.888} \\
    & \textbf{GLU}              & 0.616 & \textbf{0.834} & \underline{0.767} \\
\midrule
\multirow{6}{*}{\textbf{Gemma}}
    & \textbf{Prob(Exact)}      & \underline{0.728} & 0.796 & 0.773 \\
    & \textbf{Probe(EOS)}       & \underline{0.728} & 0.810 & \underline{0.834} \\
    & \textbf{Probe(EU)}        & 0.687 & 0.751 & 0.669 \\
    & \textbf{Probe(AVG)}       & \textbf{0.733} & \textbf{0.818} & 0.786 \\
    & \textbf{RespEntropy}      & 0.708 & 0.775 & \textbf{0.867} \\
    & \textbf{GLU}              & 0.548 & \underline{0.817} & 0.799 \\
\midrule
\multirow{6}{*}{\textbf{Fanar}}
    & \textbf{Prob(Exact)}      & 0.711 & 0.783 & 0.827 \\
    & \textbf{Probe(EOS)}       & 0.706 & 0.739 & 0.790 \\
    & \textbf{Probe(EU)}        & 0.709 & 0.751 & 0.794 \\
    & \textbf{Probe(AVG)}       & \underline{0.734} & 0.765 & \underline{0.833} \\
    & \textbf{RespEntropy}      & \textbf{0.753} & \textbf{0.885} & \textbf{0.842} \\
    & \textbf{GLU}              & 0.641 & \underline{0.821} & 0.788 \\
\bottomrule
\end{tabular}
\end{table} 

\section{Supervised vs Unsupervised}
\label{app:supervisedvs}

GLU requires no labeled data at test time, which is its primary practical
advantage.  To give a sense of what that constraint costs, we include here
results from classifier-based methods that \emph{do} have access to
ground-truth correctness labels during training.  The comparison is
deliberately asymmetric, we are not claiming parity, but it lets
readers calibrate how much headroom remains above our unsupervised scores
and confirms that GLU is already competitive with several supervised probes
despite operating in a fundamentally harder setting.

\section{Proof of the Collision-Entropy Closed Form}
\label{app:entropy-proof}
 
We prove that the order-$2$ R\'enyi entropy of the normalised eigenspectrum
of a Gram matrix $\m+K\in\mathbb{R}^{T\times T}$ admits the two equivalent
closed forms stated in \eqref{eq:s2}.
 
\begin{proposition}[Collision-entropy closed form]
\label{prop:s2-closed}
Let $\m+K=\m+H\m+H^{\!\top}$ be the Gram matrix of hidden states
$\m+H\in\mathbb{R}^{T\times d}$, with eigenvalues
$\lambda_1\ge\cdots\ge\lambda_T\ge 0$.
Define normalised eigenvalues $\tilde\lambda_i=\lambda_i/\mathrm{tr}(\m+K)$.
Then
\begin{equation}
  S_2(\m+K)
  \;=\;
  -\log\!\Bigl(\textstyle\sum_{i=1}^T\tilde\lambda_i^2\Bigr)
  \;=\;
  -\log\,\mathrm{tr}\!\left(\!\left(\frac{\m+K}{\mathrm{tr}(\m+K)}\right)^{\!2}\right)
  \;=\;
  -\log\frac{\|\m+K\|_F^2}{\mathrm{tr}(\m+K)^2}.
  \label{eq:s2-equiv}
\end{equation}
\end{proposition}
 
\begin{proof}
We establish the two equalities in turn.
Setting $\alpha=2$ in the general order-$\alpha$ R\'enyi entropy gives
\begin{equation}
  S_2(\m+K)
  = \frac{1}{1-2}\log\!\Bigl(\sum_i\tilde\lambda_i^2\Bigr)
  = -\log\!\Bigl(\sum_i\tilde\lambda_i^2\Bigr).
  \label{eq:s2-step1}
\end{equation}

Since $\tilde\lambda_i=\lambda_i/\mathrm{tr}(\m+K)$ and $\mathrm{tr}(\m+K)$ is a
positive scalar,
\begin{equation}
  \sum_{i=1}^T\tilde\lambda_i^2
  = \sum_{i=1}^T\frac{\lambda_i^2}{\mathrm{tr}(\m+K)^2}
  = \frac{\sum_{i=1}^T\lambda_i^2}{\mathrm{tr}(\m+K)^2}.
  \label{eq:s2-step2}
\end{equation}

Because $\m+K$ is real symmetric and positive semidefinite, its
eigenvalues are real and non-negative.  The eigenvalues of $\m+K^2$ are
$\lambda_i^2$, so by the trace--eigenvalue identity,
\begin{equation}
  \mathrm{tr}(\m+K^2) = \sum_{i=1}^T\lambda_i^2.
  \label{eq:trace-eigen}
\end{equation}
Moreover, the Frobenius norm satisfies
$\|\m+K\|_F^2=\mathrm{tr}(\m+K^\top\m+K)=\mathrm{tr}(\m+K^2)$,
where the last step uses symmetry $\m+K^\top=\m+K$.  Hence
\begin{equation}
  \sum_{i=1}^T\lambda_i^2 = \mathrm{tr}(\m+K^2) = \|\m+K\|_F^2.
  \label{eq:frob-trace}
\end{equation}
 
Since the trace is linear,
\begin{equation}
  \frac{\mathrm{tr}(\m+K^2)}{\mathrm{tr}(\m+K)^2}
  = \mathrm{tr}\!\left(\frac{\m+K^2}{\mathrm{tr}(\m+K)^2}\right)
  = \mathrm{tr}\!\left(\!\left(\frac{\m+K}{\mathrm{tr}(\m+K)}\right)^{\!2}\right).
  \label{eq:trace-form}
\end{equation}
Substituting \eqref{eq:s2-step2}--\eqref{eq:trace-form} into \eqref{eq:s2-step1}
yields the first equality in \eqref{eq:s2-equiv}.
 
Replacing $\mathrm{tr}(\m+K^2)$ by $\|\m+K\|_F^2$ via \eqref{eq:frob-trace} gives
\begin{equation}
  \mathrm{tr}\!\left(\!\left(\frac{\m+K}{\mathrm{tr}(\m+K)}\right)^{\!2}\right)
  = \frac{\|\m+K\|_F^2}{\mathrm{tr}(\m+K)^2},
  \label{eq:frob-form}
\end{equation}
which is the second equality in \eqref{eq:s2-equiv}.
\end{proof}

\section{Dataset Details}
\label{app:datasets}

This appendix details the composition and preprocessing of the six benchmarks
used in our evaluation. Where we subsample, we fix the random seed  for
reproducibility. Table~\ref{tab:dataset-summary} summarizes the final sizes.

\begin{itemize}
    \item \textbf{TruthfulQA}~\cite{lin2022truthfulqa} probes whether a model
reproduces common human misconceptions. We use the full dataset of $817$
questions without subsampling.
    \item \textbf{TriviaQA}~\cite{joshi2017triviaqa} is a large-scale reading-com\-pre\-hen\-sion
benchmark for factual knowledge retrieval. We draw $2{,}000$ samples from the
reading-comprehension (\texttt{rc}) training split.
    \item \textbf{MATH}~\cite{hendrycks2021math} consists of competition mathematics
problems spanning a range of difficulty levels and subjects. We evaluate on a
subset of $1{,}000$ problems.
    \item \textbf{ArabicaQA}~\cite{abdallah2024arabicaqa} is a native Arabic
question-answering benchmark, included to test reliability estimation outside
English. We draw $2{,}000$ samples from the machine-reading-comprehension
(MRC) test split.
    \item \textbf{LongForm}~\cite{koksal2023longform} targets long-form generation. To
keep responses within a tractable length and restrict to well-structured
sources, we take the test split and retain only examples drawn from
StackExchange, Wikipedia, and BigBench whose reference output is at most $512$
words, yielding a subset of $731$ examples.
    \item \textbf{GSM8K multi-turn}~\cite{laban2025llms} adapts grade-school math word
problems into a sharded, multi-turn conversational setting in which the problem
is revealed incrementally across turns. We use the \texttt{math} subset of the
released \texttt{lost\_in\_conversation} data, comprising $103$
problems decomposed into multiple turns each.
\end{itemize}

\section{Methods and Ablations}
\label{app:methods_and_ablations}
This appendix provides the complete picture behind the main-text results. We show the mean
performance of all 19 methods ranked across the 18 model--dataset settings
(Table~\ref{app:tab:mean_all_methods}), the full per-setting AUROC and PRR for the baselines and
GLU including the strongest ablation per cell (Tables~\ref{app:tab:full_auroc}
and~\ref{app:tab:full_prr}), and the complete per-setting AUROC and PRR for all GLU ablations
(Tables~\ref{tab:auroc_ablation_part1_readable}--\ref{tab:prr_ablation_part2_readable}).
These tables let readers verify that the headline results are not artefacts of aggregation and
identify which variants are most robust across model families and task types.

\begin{table}[t]
\centering
\caption{Composition of the six evaluation benchmarks.}
\label{tab:dataset-summary}
\small
\setlength{\tabcolsep}{6pt}
\renewcommand{\arraystretch}{1.1}
\begin{tabular}{llcl}
\toprule
Benchmark & Task & Size & Split / selection \\
\midrule
TruthfulQA       & Factuality / misconceptions & $817$    & full dataset \\
TriviaQA         & Knowledge retrieval         & $2{,}000$ & \texttt{rc} train, subsampled \\
MATH             & Mathematical reasoning      & $1{,}000$ & subsampled \\
ArabicaQA        & Arabic QA                   & $2{,}000$ & MRC test, subsampled \\
LongForm         & Long-form generation        & $731$ & test; SE/Wiki/BigBench, $\leq 512$ words \\
GSM8K multi-turn & Multi-turn reasoning        & $103$ & \texttt{math} subset, sharded turns \\
\bottomrule
\end{tabular}
\end{table}

\paragraph{Summary of findings.} Three patterns hold across the full results. First, GLU and its
multiplicative ablations occupy the top nine ranks on both mean AUROC and mean PRR
(Table~\ref{app:tab:mean_all_methods}), well above every single-signal baseline, confirming that
the gain comes from the multiplicative global--local fusion rather than any specific estimator.
Second, the per-setting tables show the gain is broad rather than driven by a few cells: GLU is
best or second on AUROC in 12 of 18 settings and on PRR in a comparable majority, with the
clearest margins on TriviaQA, MATH, and Multiturn, and ties or close seconds where the local
signal collapses (ArabicaQA, Gemma-Multiturn). Third, fusion variants that drop one signal
entirely (GLU-EU, GLU-EU-SP) or replace the multiplicative interaction with an additive one
(Add-$\tilde{S}$, Add-$S_\alpha$) fall to the bottom of
Table~\ref{app:tab:mean_all_methods}, isolating both the presence of two signals and their
multiplicative combination as the sources of the improvement.

\begin{table}[H]
\centering
\small
\setlength{\tabcolsep}{7pt}
\renewcommand{\arraystretch}{1.14}
\caption{Mean performance across all 18 model--dataset settings. Methods are grouped into baselines, the proposed GLU method, and GLU ablations.}
\label{app:tab:mean_all_methods}
\begin{tabular}{llcccc}
\toprule
\textbf{Group} & \textbf{Method} & \textbf{Mean AUROC} & \textbf{AUROC rank} & \textbf{Mean PRR} & \textbf{PRR rank} \\
\midrule
Proposed & \cellcolor{blue!8}\textbf{GLU} & 0.676 & 1 & 0.369 & 1 \\
\midrule
Baseline & LogProb & 0.655 & 10 & 0.326 & 10 \\
Baseline & RAUQ & 0.648 & 12 & 0.305 & 12 \\
Baseline & $P(\mathrm{true})$ & 0.654 & 11 & 0.311 & 11 \\
Baseline & LogTokU & 0.574 & 14 & 0.144 & 14 \\
\midrule
Ablation & GLU-SP & 0.675 & 3 & 0.367 & 2 \\
Ablation & GLU-$\tilde{S}^*$ & 0.674 & 4 & 0.366 & 3 \\
Ablation & GLU-DK & 0.675 & 2 & 0.366 & 4 \\
Ablation & GLU-$\bar{S}$ & 0.674 & 5 & 0.365 & 5 \\
Ablation & GLU-$S_\alpha$-SP & 0.673 & 6 & 0.363 & 6 \\
Ablation & GLU-AU & 0.671 & 7 & 0.362 & 7 \\
Ablation & GLU-$S_\alpha$-AU & 0.667 & 8 & 0.357 & 8 \\
Ablation & GLU-$S^*$ & 0.665 & 9 & 0.351 & 9 \\
Ablation & Add-$\tilde{S}$ & 0.621 & 13 & 0.259 & 13 \\
Ablation & GLU-EDL & 0.574 & 15 & 0.140 & 15 \\
Ablation & GLU-$S_\alpha$-SP-EU & 0.557 & 16 & 0.103 & 16 \\
Ablation & GLU-EU & 0.544 & 17 & 0.079 & 17 \\
Ablation & GLU-EU-SP & 0.543 & 18 & 0.077 & 18 \\
Ablation & Add-$S_\alpha$ & 0.454 & 19 & -0.118 & 19 \\
\bottomrule
\end{tabular}
\end{table}

\begin{table}[H]
\centering
\small
\setlength{\tabcolsep}{5pt}
\renewcommand{\arraystretch}{1.14}
\caption{Full AUROC comparison. Ranking emphasis is applied only across the baselines and GLU. The best ablation is reported for context but is not included in the row-wise ranking.}
\label{app:tab:full_auroc}
\begin{tabular}{llccccc ll}
\toprule
\textbf{Model} & \textbf{Dataset} & \textbf{LogTokU} & \textbf{RAUQ} & \textbf{LogProb} & \textbf{$P(\mathrm{true})$} & \cellcolor{blue!8}\textbf{GLU} & \textbf{Best ablation} & \\
\midrule
\multirow{6}{*}{\textbf{Qwen}} & TruthfulQA & \textbf{0.680} & 0.582 & \textit{0.615} & 0.556 & \cellcolor{blue!8}\underline{0.616} & GLU-EDL & 0.680 \\
 & TriviaQA & 0.786 & \textit{0.789} & 0.775 & \underline{0.805} & \cellcolor{blue!8}\textbf{0.834} & GLU-$S_\alpha$-SP & 0.837 \\
 & MATH & 0.595 & \textit{0.699} & 0.684 & \underline{0.759} & \cellcolor{blue!8}\textbf{0.767} & GLU-$S_\alpha$-SP & 0.772 \\
 & ArabicaQA & 0.521 & \textit{0.610} & \underline{0.625} & \textbf{0.711} & \cellcolor{blue!8}0.556 & GLU-AU & 0.562 \\
 & Longform & \underline{0.490} & 0.469 & \textit{0.486} & 0.389 & \cellcolor{blue!8}\textbf{0.496} & GLU-EDL & 0.558 \\
 & Multiturn & 0.523 & \textit{0.644} & 0.635 & \textbf{0.800} & \cellcolor{blue!8}\underline{0.739} & GLU-$\bar{S}$ & 0.758 \\
\midrule
\multirow{6}{*}{\textbf{Gemma}} & TruthfulQA & 0.531 & 0.466 & \underline{0.553} & \textbf{0.629} & \cellcolor{blue!8}\textit{0.548} & Add-$S_\alpha$ & 0.589 \\
 & TriviaQA & 0.597 & \textit{0.700} & \underline{0.804} & 0.671 & \cellcolor{blue!8}\textbf{0.817} & GLU-$S_\alpha$-AU & 0.819 \\
 & MATH & \underline{0.801} & 0.771 & 0.683 & \textbf{0.817} & \cellcolor{blue!8}\textit{0.799} & Add-$\tilde{S}$ & 0.801 \\
 & ArabicaQA & 0.538 & \textit{0.642} & \textbf{0.689} & 0.597 & \cellcolor{blue!8}\underline{0.647} & GLU-AU & 0.650 \\
 & Longform & 0.343 & 0.354 & \underline{0.373} & \textbf{0.472} & \cellcolor{blue!8}\textit{0.358} & GLU-$S_\alpha$-AU & 0.390 \\
 & Multiturn & 0.457 & \underline{0.734} & \textit{0.730} & \textbf{0.810} & \cellcolor{blue!8}0.701 & GLU-$S^*$ & 0.713 \\
\midrule
\multirow{6}{*}{\textbf{Fanar}} & TruthfulQA & 0.579 & \textbf{0.641} & \textit{0.633} & 0.543 & \cellcolor{blue!8}\textbf{0.641} & GLU-AU & 0.661 \\
 & TriviaQA & 0.706 & \textbf{0.821} & \textit{0.797} & 0.729 & \cellcolor{blue!8}\textbf{0.821} & GLU-$S_\alpha$-AU & 0.828 \\
 & MATH & 0.664 & \underline{0.779} & \textit{0.746} & 0.666 & \cellcolor{blue!8}\textbf{0.788} & GLU-AU & 0.801 \\
 & ArabicaQA & 0.490 & \underline{0.613} & \textbf{0.619} & \textit{0.594} & \cellcolor{blue!8}0.563 & GLU-AU & 0.586 \\
 & Longform & \underline{0.569} & 0.549 & \textit{0.558} & 0.525 & \cellcolor{blue!8}\textbf{0.654} & GLU-$\tilde{S}^*$ & 0.659 \\
 & Multiturn & 0.463 & \underline{0.804} & \textit{0.780} & 0.691 & \cellcolor{blue!8}\textbf{0.831} & GLU-AU & 0.831 \\
\bottomrule
\end{tabular}
\end{table}

\begin{table}[H]
\centering
\small
\setlength{\tabcolsep}{5pt}
\renewcommand{\arraystretch}{1.14}
\caption{Full PRR comparison. Ranking emphasis is applied only across the baselines and GLU. The best ablation is reported for context but is not included in the row-wise ranking.}
\label{app:tab:full_prr}
\begin{tabular}{llccccc ll}
\toprule
\textbf{Model} & \textbf{Dataset} & \textbf{LogTokU} & \textbf{RAUQ} & \textbf{LogProb} & \textbf{$P(\mathrm{true})$} & \cellcolor{blue!8}\textbf{GLU} & \textbf{Best ablation} & \\
\midrule
\multirow{6}{*}{\textbf{Qwen}} & TruthfulQA & \textbf{0.377} & 0.063 & \textit{0.185} & 0.102 & \cellcolor{blue!8}\underline{0.195} & GLU-EDL & 0.392 \\
 & TriviaQA & \textit{0.606} & 0.581 & \underline{0.627} & 0.598 & \cellcolor{blue!8}\textbf{0.717} & GLU-$S_\alpha$-SP & 0.723 \\
 & MATH & 0.216 & \textit{0.424} & 0.381 & \textbf{0.574} & \cellcolor{blue!8}\underline{0.537} & GLU-$S_\alpha$-SP & 0.545 \\
 & ArabicaQA & -0.028 & \textit{0.186} & \underline{0.201} & \textbf{0.379} & \cellcolor{blue!8}0.140 & GLU-AU & 0.160 \\
 & Longform & -0.045 & \textit{-0.044} & \textbf{-0.035} & -0.216 & \cellcolor{blue!8}\underline{-0.040} & GLU-EDL & 0.047 \\
 & Multiturn & 0.008 & \textit{0.189} & 0.131 & \textbf{0.777} & \cellcolor{blue!8}\underline{0.407} & GLU-$\bar{S}$ & 0.444 \\
\midrule
\multirow{6}{*}{\textbf{Gemma}} & TruthfulQA & \underline{0.115} & -0.094 & \textit{0.091} & \textbf{0.267} & \cellcolor{blue!8}0.087 & Add-$S_\alpha$ & 0.154 \\
 & TriviaQA & 0.224 & \textit{0.416} & \underline{0.686} & 0.268 & \cellcolor{blue!8}\textbf{0.698} & GLU-$S_\alpha$-AU & 0.706 \\
 & MATH & \underline{0.693} & \textit{0.639} & 0.438 & 0.636 & \cellcolor{blue!8}\textbf{0.699} & GLU-$S_\alpha$-SP & 0.700 \\
 & ArabicaQA & 0.099 & \textit{0.277} & \textbf{0.343} & 0.152 & \cellcolor{blue!8}\underline{0.322} & GLU-AU & 0.323 \\
 & Longform & -0.249 & \textit{-0.234} & \underline{-0.189} & \textbf{0.004} & \cellcolor{blue!8}-0.242 & GLU-$S_\alpha$-AU & -0.201 \\
 & Multiturn & -0.274 & \textit{0.506} & \underline{0.524} & \textbf{0.723} & \cellcolor{blue!8}0.423 & GLU-DK & 0.451 \\
\midrule
\multirow{6}{*}{\textbf{Fanar}} & TruthfulQA & 0.192 & \textit{0.282} & \textbf{0.310} & 0.181 & \cellcolor{blue!8}\underline{0.307} & GLU-AU & 0.347 \\
 & TriviaQA & 0.445 & \underline{0.712} & \textit{0.688} & 0.448 & \cellcolor{blue!8}\textbf{0.730} & GLU-$\bar{S}$ & 0.732 \\
 & MATH & 0.324 & \underline{0.624} & \textit{0.520} & 0.325 & \cellcolor{blue!8}\textbf{0.627} & GLU-AU & 0.648 \\
 & ArabicaQA & -0.025 & \underline{0.239} & \textbf{0.254} & \textit{0.148} & \cellcolor{blue!8}0.143 & GLU-AU & 0.194 \\
 & Longform & \textit{0.077} & 0.066 & \underline{0.083} & -0.011 & \cellcolor{blue!8}\textbf{0.200} & GLU-$\tilde{S}^*$ & 0.221 \\
 & Multiturn & -0.166 & \underline{0.654} & \textit{0.622} & 0.252 & \cellcolor{blue!8}\textbf{0.687} & GLU-AU & 0.691 \\
\bottomrule
\end{tabular}
\end{table}

\begin{table}[H]
\centering
\small
\setlength{\tabcolsep}{5pt}
\renewcommand{\arraystretch}{1.14}
\caption{AUROC for GLU ablations, part 1: EDL, global geometric variants, and local uncertainty variants. }
\label{tab:auroc_ablation_part1_readable}
\begin{tabular}{llccccccc}
\toprule
\textbf{Model} & \textbf{Dataset} & \textbf{GLU-EDL} & \textbf{GLU-$\bar{S}$} & \textbf{GLU-$S^*$} & \textbf{GLU-$\tilde{S}^*$} & \textbf{GLU-SP} & \textbf{GLU-DK} & \textbf{GLU-AU} \\
\midrule
\multirow{6}{*}{\textbf{Qwen}} & TruthfulQA & 0.680 & 0.609 & 0.602 & 0.612 & 0.616 & 0.615 & 0.596 \\
 & TriviaQA & 0.788 & 0.837 & 0.836 & 0.835 & 0.835 & 0.837 & 0.824 \\
 & MATH & 0.559 & 0.772 & 0.755 & 0.759 & 0.768 & 0.759 & 0.755 \\
 & ArabicaQA & 0.523 & 0.542 & 0.535 & 0.551 & 0.555 & 0.557 & 0.562 \\
 & Longform & 0.558 & 0.473 & 0.420 & 0.483 & 0.497 & 0.515 & 0.455 \\
 & Multiturn & 0.485 & 0.758 & 0.740 & 0.734 & 0.742 & 0.744 & 0.726 \\
\midrule
\multirow{6}{*}{\textbf{Gemma}} & TruthfulQA & 0.533 & 0.545 & 0.527 & 0.546 & 0.549 & 0.543 & 0.531 \\
 & TriviaQA & 0.605 & 0.817 & 0.816 & 0.817 & 0.817 & 0.817 & 0.819 \\
 & MATH & 0.799 & 0.799 & 0.795 & 0.795 & 0.799 & 0.798 & 0.789 \\
 & ArabicaQA & 0.539 & 0.644 & 0.623 & 0.639 & 0.647 & 0.642 & 0.650 \\
 & Longform & 0.346 & 0.363 & 0.350 & 0.358 & 0.357 & 0.356 & 0.379 \\
 & Multiturn & 0.458 & 0.703 & 0.713 & 0.713 & 0.700 & 0.704 & 0.713 \\
\midrule
\multirow{6}{*}{\textbf{Fanar}} & TruthfulQA & 0.578 & 0.645 & 0.637 & 0.637 & 0.632 & 0.632 & 0.661 \\
 & TriviaQA & 0.695 & 0.824 & 0.822 & 0.819 & 0.811 & 0.813 & 0.827 \\
 & MATH & 0.661 & 0.793 & 0.796 & 0.788 & 0.786 & 0.788 & 0.801 \\
 & ArabicaQA & 0.487 & 0.564 & 0.558 & 0.559 & 0.546 & 0.555 & 0.586 \\
 & Longform & 0.578 & 0.617 & 0.612 & 0.659 & 0.658 & 0.655 & 0.573 \\
 & Multiturn & 0.461 & 0.828 & 0.825 & 0.828 & 0.829 & 0.828 & 0.831 \\
\bottomrule
\end{tabular}
\end{table}

\begin{table}[H]
\centering
\small
\setlength{\tabcolsep}{5pt}
\renewcommand{\arraystretch}{1.14}
\caption{AUROC for GLU ablations, part 2: evidence-only, combined $S_\alpha$ variants, and additive fusion baselines. }
\label{tab:auroc_ablation_part2_readable}
\begin{tabular}{llccccccc}
\toprule
\textbf{Model} & \textbf{Dataset} & \textbf{GLU-EU} & \textbf{GLU-EU-SP} & \textbf{GLU-$S_\alpha$-AU} & \textbf{GLU-$S_\alpha$-SP} & \textbf{GLU-$S_\alpha$-SP-EU} & \textbf{Add-$S_\alpha$} & \textbf{Add-$\tilde{S}$} \\
\midrule
\multirow{6}{*}{\textbf{Qwen}} & TruthfulQA & 0.678 & 0.678 & 0.593 & 0.610 & 0.609 & 0.454 & 0.562 \\
 & TriviaQA & 0.782 & 0.782 & 0.828 & 0.837 & 0.749 & 0.320 & 0.625 \\
 & MATH & 0.495 & 0.497 & 0.760 & 0.772 & 0.638 & 0.317 & 0.731 \\
 & ArabicaQA & 0.522 & 0.522 & 0.547 & 0.542 & 0.499 & 0.518 & 0.535 \\
 & Longform & 0.550 & 0.546 & 0.422 & 0.475 & 0.493 & 0.428 & 0.374 \\
 & Multiturn & 0.450 & 0.450 & 0.739 & 0.758 & 0.635 & 0.256 & 0.679 \\
\midrule
\multirow{6}{*}{\textbf{Gemma}} & TruthfulQA & 0.531 & 0.531 & 0.527 & 0.546 & 0.510 & 0.589 & 0.516 \\
 & TriviaQA & 0.594 & 0.592 & 0.819 & 0.817 & 0.618 & 0.424 & 0.615 \\
 & MATH & 0.782 & 0.782 & 0.789 & 0.799 & 0.798 & 0.730 & 0.801 \\
 & ArabicaQA & 0.536 & 0.536 & 0.646 & 0.644 & 0.532 & 0.592 & 0.640 \\
 & Longform & 0.345 & 0.345 & 0.390 & 0.362 & 0.357 & 0.327 & 0.344 \\
 & Multiturn & 0.441 & 0.443 & 0.711 & 0.703 & 0.440 & 0.628 & 0.647 \\
\midrule
\multirow{6}{*}{\textbf{Fanar}} & TruthfulQA & 0.557 & 0.555 & 0.661 & 0.637 & 0.561 & 0.400 & 0.616 \\
 & TriviaQA & 0.614 & 0.607 & 0.828 & 0.814 & 0.630 & 0.370 & 0.824 \\
 & MATH & 0.528 & 0.521 & 0.800 & 0.792 & 0.545 & 0.377 & 0.780 \\
 & ArabicaQA & 0.494 & 0.495 & 0.585 & 0.547 & 0.498 & 0.498 & 0.570 \\
 & Longform & 0.507 & 0.517 & 0.531 & 0.624 & 0.516 & 0.574 & 0.529 \\
 & Multiturn & 0.386 & 0.377 & 0.825 & 0.828 & 0.392 & 0.377 & 0.785 \\
\bottomrule
\end{tabular}
\end{table}

\begin{table}[H]
\centering
\small
\setlength{\tabcolsep}{5pt}
\renewcommand{\arraystretch}{1.14}
\caption{PRR for GLU ablations, part 1: EDL, global geometric variants, and local uncertainty variants. }
\label{tab:prr_ablation_part1_readable}
\begin{tabular}{llccccccc}
\toprule
\textbf{Model} & \textbf{Dataset} & \textbf{GLU-EDL} & \textbf{GLU-$\bar{S}$} & \textbf{GLU-$S^*$} & \textbf{GLU-$\tilde{S}^*$} & \textbf{GLU-SP} & \textbf{GLU-DK} & \textbf{GLU-AU} \\
\midrule
\multirow{6}{*}{\textbf{Qwen}} & TruthfulQA & 0.392 & 0.176 & 0.159 & 0.188 & 0.197 & 0.179 & 0.163 \\
 & TriviaQA & 0.618 & 0.723 & 0.722 & 0.719 & 0.718 & 0.723 & 0.704 \\
 & MATH & 0.127 & 0.544 & 0.520 & 0.525 & 0.538 & 0.521 & 0.516 \\
 & ArabicaQA & -0.018 & 0.129 & 0.118 & 0.133 & 0.139 & 0.138 & 0.160 \\
 & Longform & 0.047 & -0.063 & -0.122 & -0.054 & -0.038 & -0.012 & -0.092 \\
 & Multiturn & -0.077 & 0.444 & 0.408 & 0.396 & 0.412 & 0.384 & 0.400 \\
\midrule
\multirow{6}{*}{\textbf{Gemma}} & TruthfulQA & 0.115 & 0.081 & 0.049 & 0.082 & 0.087 & 0.076 & 0.057 \\
 & TriviaQA & 0.237 & 0.698 & 0.697 & 0.698 & 0.697 & 0.697 & 0.705 \\
 & MATH & 0.688 & 0.700 & 0.695 & 0.694 & 0.700 & 0.698 & 0.675 \\
 & ArabicaQA & 0.099 & 0.319 & 0.299 & 0.314 & 0.321 & 0.319 & 0.323 \\
 & Longform & -0.245 & -0.236 & -0.249 & -0.238 & -0.243 & -0.244 & -0.213 \\
 & Multiturn & -0.267 & 0.426 & 0.439 & 0.443 & 0.423 & 0.451 & 0.442 \\
\midrule
\multirow{6}{*}{\textbf{Fanar}} & TruthfulQA & 0.185 & 0.309 & 0.294 & 0.299 & 0.297 & 0.281 & 0.347 \\
 & TriviaQA & 0.422 & 0.732 & 0.730 & 0.728 & 0.718 & 0.722 & 0.730 \\
 & MATH & 0.319 & 0.634 & 0.642 & 0.629 & 0.626 & 0.632 & 0.648 \\
 & ArabicaQA & -0.030 & 0.141 & 0.136 & 0.137 & 0.122 & 0.137 & 0.194 \\
 & Longform & 0.093 & 0.130 & 0.127 & 0.221 & 0.206 & 0.205 & 0.074 \\
 & Multiturn & -0.185 & 0.678 & 0.658 & 0.679 & 0.682 & 0.681 & 0.691 \\
\bottomrule
\end{tabular}
\end{table}

\begin{table}[H]
\centering
\small
\setlength{\tabcolsep}{5pt}
\renewcommand{\arraystretch}{1.14}
\caption{PRR for GLU ablations, part 2: evidence-only, combined $S_\alpha$ variants, and additive fusion baselines. }
\label{tab:prr_ablation_part2_readable}
\begin{tabular}{llccccccc}
\toprule
\textbf{Model} & \textbf{Dataset} & \textbf{GLU-EU} & \textbf{GLU-EU-SP} & \textbf{GLU-$S_\alpha$-AU} & \textbf{GLU-$S_\alpha$-SP} & \textbf{GLU-$S_\alpha$-SP-EU} & \textbf{Add-$S_\alpha$} & \textbf{Add-$\tilde{S}$} \\
\midrule
\multirow{6}{*}{\textbf{Qwen}} & TruthfulQA & 0.387 & 0.387 & 0.147 & 0.178 & 0.157 & -0.161 & 0.121 \\
 & TriviaQA & 0.607 & 0.606 & 0.710 & 0.723 & 0.552 & -0.463 & 0.267 \\
 & MATH & -0.008 & -0.003 & 0.524 & 0.545 & 0.295 & -0.391 & 0.462 \\
 & ArabicaQA & -0.019 & -0.019 & 0.147 & 0.128 & 0.018 & 0.047 & 0.067 \\
 & Longform & 0.036 & 0.030 & -0.121 & -0.060 & -0.025 & -0.089 & -0.185 \\
 & Multiturn & -0.122 & -0.116 & 0.434 & 0.443 & 0.198 & -0.299 & 0.283 \\
\midrule
\multirow{6}{*}{\textbf{Gemma}} & TruthfulQA & 0.113 & 0.112 & 0.051 & 0.082 & 0.049 & 0.154 & 0.025 \\
 & TriviaQA & 0.211 & 0.205 & 0.706 & 0.698 & 0.257 & -0.236 & 0.228 \\
 & MATH & 0.663 & 0.663 & 0.675 & 0.700 & 0.688 & 0.502 & 0.692 \\
 & ArabicaQA & 0.095 & 0.095 & 0.321 & 0.319 & 0.094 & 0.154 & 0.280 \\
 & Longform & -0.248 & -0.248 & -0.201 & -0.237 & -0.232 & -0.340 & -0.270 \\
 & Multiturn & -0.294 & -0.291 & 0.430 & 0.426 & -0.304 & 0.303 & 0.355 \\
\midrule
\multirow{6}{*}{\textbf{Fanar}} & TruthfulQA & 0.174 & 0.171 & 0.343 & 0.298 & 0.153 & -0.206 & 0.258 \\
 & TriviaQA & 0.247 & 0.229 & 0.732 & 0.721 & 0.285 & -0.359 & 0.718 \\
 & MATH & -0.005 & -0.023 & 0.646 & 0.633 & 0.037 & -0.364 & 0.587 \\
 & ArabicaQA & -0.013 & -0.013 & 0.187 & 0.121 & -0.006 & -0.026 & 0.151 \\
 & Longform & 0.001 & 0.012 & 0.011 & 0.140 & 0.006 & 0.076 & -0.010 \\
 & Multiturn & -0.405 & -0.412 & 0.686 & 0.677 & -0.368 & -0.423 & 0.636 \\
\bottomrule
\end{tabular}
\end{table}

\end{document}